\documentclass[10pt,twoside]{article}

\usepackage{times}
\usepackage[utf8]{inputenc}
\usepackage[T1]{fontenc}
\usepackage{icomma}
\usepackage{graphicx}
\usepackage{algorithm2e}[french]
\usepackage{amsmath,amssymb,amsthm}  
\usepackage{bbm}
\usepackage{bm}  
\usepackage{graphicx}
\usepackage{ulem}
\usepackage{hyperref}
\graphicspath{ {.} }

\usepackage{jeptaln2024}
\usepackage[french]{babel} 

\RestyleAlgo{ruled}

\SetKwRepeat{Do}{do}{Tant que}
\SetKwComment{Comment}{/* }{ */}
\SetKw{Or}{\itshape ou}
\SetKw{And}{\itshape et}
\SetKwInput{KwRes}{R\'esultat}%
\SetKwIF{If}{ElseIf}{Else}{si}{alors}{sinon si}{sinon}{fin si}
\SetKwFor{While}{tant que}{faire}{fin tq}%
\SetKwFor{For}{pour}{faire}{fin pour}%
\SetKwInput{KwResult}{Résultat}%
\SetKwInput{KwData}{Données}%
\SetKw{KwRet}{retroune}%
\SetKw{Return}{retourne}%

\DeclareMathOperator*{\argmax}{arg\,max}

\newcommand{\mlevt}[1]{\texttt{TM$^{#1}$-LevT}}

\usepackage[colorinlistoftodos]{todonotes}

\newtheorem{definition}{Définition}
\newtheorem{lemma}{Lemme}
\newtheorem{theorem}{Théorème}

\title{Optimiser le choix des exemples pour la traduction automatique augmentée par des mémoires de traduction}

\author{Maxime Bouthors\up{1, 2}\quad Josep Crego\up{1}\quad François Yvon\up{2}\\
  {\small
    (1) Sorbonne Université, CNRS, ISIR, F-75005 Paris, France \\ 
    (2) ChapsVision, 92150, Suresnes \\ 
    \texttt{
      bouthors@isir.upmc.fr, jcrego@chapsvision.com, yvon@isir.upmc.fr \\ 
    }}}

\begin{document}
\maketitle
\resume{%
	La traduction neuronale à partir d'exemples s'appuie sur l'exploitation d'une mémoire de traduction contenant des exemples similaires aux phrases à traduire. Ces exemples sont utilisés pour conditionner les prédictions d'un décodeur neuronal. Nous nous intéressons à l'amélioration du système qui effectue l'étape de recherche des phrases similaires, l'architecture du décodeur neuronal étant fixée et reposant ici sur un modèle explicite d'édition, le Transformeur \og multi-Levenshtein \fg{}. Le problème considéré consiste à trouver un ensemble optimal d'exemples similaires, c'est-à-dire qui couvre maximalement la phrase source. En nous appuyant sur la théorie des fonctions sous-modulaires, nous explorons de nouveaux algorithmes pour optimiser cette couverture et évaluons les améliorations de performances auxquelles ils mènent pour la tâche de traduction automatique.
}%

\abstract{Optimizing example selection for retrieval-augmented machine translation with translation memories}{
	Retrieval-augmented machine translation leverages examples from a translation memory by retrieving similar instances. These examples are used to condition the predictions of a neural decoder. We aim to improve the upstream retrieval step and consider a fixed downstream edit-based model: the multi-Levenshtein Transformer. The task consists of finding a set of examples that maximizes the overall coverage of the source sentence. To this end, we rely on the theory of submodular functions and explore new algorithms to optimize this coverage. We evaluate the resulting performance gains for the machine translation task.\\
}

\motsClefs
{Traduction Automatique, Recherche d'Information, Mémoires de Traduction, Fonctions Sous-Modulaires, Traduction à partir d'Exemples}
{Machine Translation, Information Retrieval, Translation Memories, Submodularity, Example-based Translation}

\section{Introduction}
De nombreux travaux récents s'intéressent à la génération augmentée par des exemples \citep{li-etal-2023-survey}. En traduction, l'utilisation d'exemples remonte aux méthodes de traduction assistée par ordinateur par des traducteurs professionnels \citep{bowker-2002-computer}: éditer des segments très similaires à la phrase de référence permet d'accélérer la traduction. Cette idée est au fondement des méthodes basées sur des exemples \cite{nagao-1984-framework,somers-1999-review,carl-etal-2004-recent}. Elle est adaptée aux systèmes statistiques par \cite{koehn-senellart-2010-convergence} et plus récemment aux méthodes neuronales.

Il existe en fait de nombreuses manières d'exploiter les exemples :
la traduction conditionnelle qui introduit un système d'attention sur les exemples \citep{gu-etal-2018-search,bulte-tezcan-2019-neural,hoang-etal-2022-improving}; l'affinage \og léger\fg{} sur un ensemble d'exemples pour faire de la micro-adaptation \citep{farajian-etal-2017-multi}; les méthodes intégrant des exemples dans le contexte de grands modèles de langue (LLM) génératifs multilingues (\citet{moslem-etal-2023-adaptive}, \textsl{inter alia}); l'édition directe du meilleur exemple similaire \citep{gu-etal-2019-levenshtein}. Nous discutons ces études dans la section~\ref{sec:related-work}.

Ici, nous nous intéressons au transformeur \og multi-Levenshtein\fg{} \citep{bouthors-etal-2023-towards}, un modèle d'édition qui combine $k (\geq 1)$ exemples pour calculer une traduction. Cette caractéristique le rend sensible à la qualité des exemples récupérés. En particulier, supposant fixé $k$ le nombre de phrases à récupérer, nous cherchons à répondre à la question : comment identifier un ensemble optimal de $k$ exemples ? Trouver exactement ce meilleur ensemble conditionnellement au modèle et à la phrase source est difficile, ce qui implique de considérer des heuristiques. Pour les construire, nous faisons l'hypothèse qu'un ensemble de phrases parallèles couvrant (en source) la phrase à traduire fournit des exemples couvrant (en cible) la traduction à produire. En dépit de ses limites, liées à des phénomènes linguistiques bien connus (variation lexicale, divergences morphologiques ou syntaxiques entre langues source et cible, etc.), cette hypothèse est adoptée dans les travaux de l'état-de-l'art.

Notre contribution principale est alors l'étude de plusieurs manières de définir la notion de couverture et de rechercher des $k$ meilleurs exemples dans une mémoire de traduction. En tirant parti de la théorie des fonctions sous-modulaires, dont une sous-classe correspond à une notion très générique de couverture, nous analysons dans un cadre unifié les avantages comparés de ces différentes propositions, et évaluons leur impact dans une tâche de traduction multidomaines.

\section{Travaux Connexes}
\label{sec:related-work}

De nombreux efforts pour intégrer des exemples dans la génération de textes ont été menés ces dernières années \cite{li-etal-2023-survey}.
Au-delà des améliorations de performances, la possibilité de présenter aux utilisateurs améliore la transparence des décisions qui sont prises \citep{rudin-cynthia-2019-stop}. En traduction automatique, les exemples récupérés d'une mémoire de traduction sont fournis au modèle comme un contexte supplémentaire, par exemple en concaténant le côté cible des exemples au texte source \cite{bulte-tezcan-2019-neural}, ou en tirant parti à la fois de la source et de la cible \cite{pham-etal-2020-priming}. \citet{gu-etal-2018-search,xia-etal-2019-graph,he-etal-2021-fast} considèrent des stratégies plus sophistiquées pour enrichir le contexte source.

Si beaucoup de travaux se limitent à considérer les $k \ge 1$ exemples les plus similaires, \citet{cheng-etal-2022-neural,agrawal-etal-2023-context,sia-and-duh-2023-incontext} cherchent à trouver un ensemble d'exemples complémentaires entre eux. Le premier travail utilise l'algorithme de \textit{Maximum Marginal Relevance} (MMR) \citep{goldstein-carbonell-1998-summarization}, tandis que les deux autres proposent une forme de maximisation de couverture. \citet{gupta-etal-2023-coverage} donne une formulation générale du problème de couverture, l'appliquant à l'apprentissage en contexte (\textsl{in context learning}) sur des tâches diverses.

Une autre extension de cette approche exploite des corpus monolingues, en recherchant les exemples directement dans la langue cible. \citet{cai-etal-2021-neural} proposent un modèle de recherche et de traduction unique entraîné de bout-en-bout dont la procédure de recherche translingue est optimisée pour retourner des exemples utiles pour la Traduction Automatique (TA). 

La plupart de ces travaux reposent sur des modèles de génération auto-régressifs (AR), impliquant que les exemples intégrés au contexte n'ont qu'un effet indirect sur la  sortie. L'utilisation d'une mémoire de traduction avec un décodeur non auto-régressif (NAR) est étudiée par \citet{niwa-etal-2022-nearest,xu-etal-2023-integrating,zheng-etal-2023-towards} qui adaptent le transformeur de Levenshtein \cite{gu-etal-2019-levenshtein} pour éditer directement l'exemple le plus similaire en une traduction de la phrase source. \citet{bouthors-etal-2023-towards} étendent cette technique à plusieurs exemples.

Une autre approche consiste à utiliser des similarités au niveau des contextes de génération, c'est-à-dire d'états cachés du décodage des mémoires de traduction, plutôt qu'au niveau des phrases \citep{zhang-etal-2018-guiding}. \citet{he-etal-2021-efficient,khandelwal_nearest_2021}, entre autres, utilisent des méthodes de plus proches voisins sur des contextes. La prédiction du token suivant est alors guidée par ceux trouvés dans des contextes proches. Diverses extensions sont apportées par \citet{zheng-etal-2021-adaptive,meng-etal-2022-fast,martins-etal-2022-chunkbased}.

Il est enfin difficile d'ignorer l'essor des grands modèles de langue multilingues (LLM) qui, amorcés par un contexte contenant une description de la tâche à accomplir et des exemples, peuvent générer des traductions de qualité. Cette approche a été testée sur de nombreux LLM dans le but de mettre en évidence leur capacité à traiter de multiples tâches. Pour ce qui concerne la TA, plusieurs travaux étudient l'impact du contexte d'entrée, en cherchant à optimiser le nombre d'exemples et leur sélection \citep{vilar-etal-2023-prompting,zhang-et-al-2023-prompting,hendy-etal-2023-howgood,bawden-yvon-2023-investigating}.
Voir également sur ces questions \citep{moslem-etal-2023-adaptive,mu-etal-2023-augmenting,agrawal-etal-2023-context,sia-and-duh-2023-incontext,m-etal-2023-ctqscorer}.

\section{Le modèle \og multi-Levenshtein\fg{}}


\begin{figure}[ht!]
	\centering
	\includegraphics[scale=0.6]{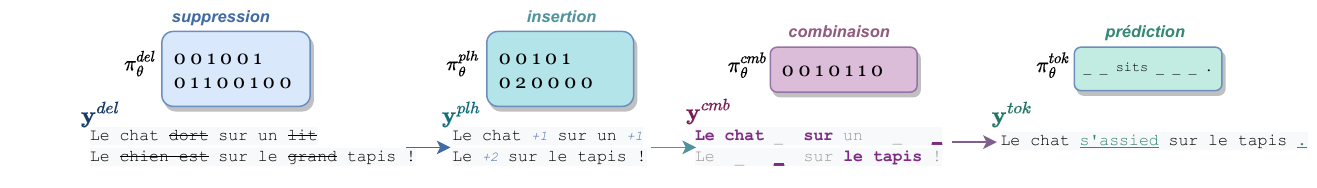}
	\caption{Première étape de décodage de \mlevt{k}. Les deux exemples qui sont édités sont $y_1$: \og le chat dort sur un lit\fg{} et $y_2$:\og le chien est sur le grand tapis\fg. Les insertions prédites à l'étape~2 (insertion) sont représentées par des entiers, puis matérialisées par des '\_' à l'étape 3 (combinaison).}
	\label{fig:edition_steps_inference}
\end{figure}

Nous nous intéressons au modèle du transformeur multi-Levenshtein, \mlevt{k}, \citep{bouthors-etal-2023-towards}, qui repose sur un modèle explicite d'édition d'exemples. L'algorithme de décodage est représenté sur la Figure~\ref{fig:edition_steps_inference}. Il s'agit d'un modèle d'édition type transformeur \citep{vaswani-etal-2017-attention} qui prend en entrée $k$ couples de phrases exemples cibles $\{y_1, \dots, y_k\}$, et les édite conditionnellement à la source $x$ en quatre étapes :
\begin{enumerate}
	\item \textbf{Suppression} : pour chaque exemple cible $y_i$, délétion de tokens;
	\item \textbf{Insertion} : pour chaque exemple cible $y_i$, insertion de tokens vides \verb|PLH| entre chaque token;
	\item \textbf{Combinaison} : combinaison des $k$ exemples cibles en s'appuyant sur un multi-alignement qui permet de déterminer, pour chaque position, l'origine (parmi $y_1 \dots y_k$) du token à conserver;
	\item \textbf{Prédiction} : pour chaque position comprenant un \verb|PLH|, prédiction du mot à insérer.
\end{enumerate}
Le calcul des associations entrées/sorties est réalisé par une architecture encodeur-décodeur non-autorégressive, dans lequel on remplace la couche de prédiction de mots habituelle par  quatre couches linéaires (une pour chaque opération) qui projettent les représentations latentes sur l'ensemble des opérations possibles. Par exemple, pour l'insertion (étape~2), pour une source $x$ et des exemples $y_1, \cdots, y_k$ :
\begin{equation}
	\texttt{insertion}^* = \argmax \texttt{Insertion}(\texttt{Décodeur}(\texttt{Encodeur}(x), y_1, \cdots, y_k))
\end{equation}

Le modèle est entraîné par apprentissage par imitation \cite{daume-etal-2009-search,ross-etal-2011-reduction}, en utilisant comme politique experte\footnote{En apprentissage par imitation, la politique \og experte\fg{} est celle que l'apprenti cherche à reproduire \cite{knyazeva-etal-2018-les}: elle indique ici les actions optimales qu'il faut effectuer pour éditer les exemples et générer la traduction de référence. Comme elle n'est pas observée dans les données d'apprentissage, il faut la calculer en s'appuyant sur diverses heuristiques.} la série d'opérations qui maximise la copie des tokens qui sont à la fois présents dans les exemples d'entrée et dans la référence. Autrement dit, cette politique optimale s'appuie sur une notion de couverture optimale. Pour déterminer cette politique, un algorithme d'alignement calcule la manière optimale de faire correspondre les exemples et la référence.

On se reportera à \citep{bouthors-etal-2023-towards} pour une présentation détaillée de cette architecture de base et de diverses extensions (réalignement, pré-apprentissage) qui lui permettent de tirer efficacement parti de plusieurs exemples similaires. L'essentiel étant de noter que dans son principe même, cette architecture est particulièrement sensible aux exemples qui sont fournis en entrée du système.

\section{Recherche d'Information dans une Mémoire de Traduction \label{sec:ritm}}

\subsection{Recherche de Phrases Similaires (RPS)}

Supposons que l'on ait accès à une mémoire de traduction $\mathcal D = \{(x_1, y_1), \cdots (x_N, y_N)\}$ et soit $x$ une phrase source que l'on cherche à traduire. Le cadre classique consiste à évaluer indépendamment chaque candidat selon un score de similarité $s(x_i, x)$ et récupérer les $k$ exemples les plus similaires.
$s$ peut être un score lexical (similarité de Jaccard, TF-IDF, BM25, distance d'édition, rappel n-gramme, etc.) ou sémantique (similarité cosinus entre deux plongements). Les premiers (Jaccard, BM25) reposent sur des algorithmes simples et servent souvent à filtrer un premier ensemble $T$ de candidats, que l'on peut ensuite évaluer avec des fonctions de comparaison plus sophistiquées.

Cependant, récupérer les $k$ meilleurs candidats peut s'avérer sous-optimal, par exemple lorsque tous ces exemples sont très similaires entre eux. Pour évaluer \emph{globalement} l'ensemble des candidats retournés par la RPS, des stratégies visant à introduire de la diversité (par exemple via l'algorithme de \textsl{Maximum Marginal Relevance}, MMR) ou des contraintes de couverture peuvent alors être déployées. Nous considérons la deuxième stratégie, la première étant documentée dans \citep{cheng-etal-2022-neural}.

\subsection{Fonctions sous-modulaires et couverture}

Par analogie aux formalisations développées dans un cadre de recherche d'information pour trouver des résultats variés \cite{lin-bilmes-2011-class,krause-golovin-2014-submodular}, nous nous appuyons sur la théorie des fonctions sous-modulaires pour formaliser la RPS basée sur la maximisation de la couverture de la phrase source. Nous commençons par rappeler quelques définitions, avant de présenter les algorithmes de sélection d'exemples.

\subsubsection{Définitions}

\begin{definition}[Sous-modularité]
	Soient $\Omega$ un ensemble et $f : 2^\Omega \to \mathbb{R}$,
	$f$ est \textbf{sous-modulaire} si $\forall X, Y$ tels que $X \subset Y \subset \Omega, \forall z \in \Omega \setminus Y$ :
	\[
		f(X \cup \{z\}) - f(X) \geq f(Y \cup \{z\}) - f(Y)
	\]
\end{definition}
Intuitivement, cette définition exprime que le rendement marginal de $f$ est décroissant : plus l'ensemble en entrée de $f$ est grand, plus les incréments de $f$ induits par l'ajout de nouveaux éléments sont faibles. Une classe de fonctions sous-modulaires bien documentée est la classe des fonctions de couverture pondérée \cite{krause-golovin-2014-submodular}.
\begin{definition}[Couverture pondérée]
	\label{def:couv-ponderee}
	Soit $N \in \mathbb{N}$ et $v^{(n)}(z)_{n \in [1, N]}$ une séquence de poids réels associée à $z \in \Omega$, correspondant à des \emph{aspects} (de $x$): $v^{(n)}(z)$ évalue à quel point $z$ couvre l'\emph{aspects}~$n$. Une \textbf{fonction de couverture pondérée} associe à un sous-ensemble $Z$ de $\Omega$:
	\[
		f(Z) = \sum_{n=1}^N \max_{z \in Z} v^{(n)}(z)
	\]
\end{definition}

Dans notre application, $\Omega$ est un ensemble d'exemples (source et références jointes) et $N$ dénombre des \emph{aspects} importants de la source $x$ qu'il faut couvrir. Cet ensemble d'\emph{aspects} peut être le sac-de-mots associé à $x=(x_1, \cdots, x_N)$, les indices de la séquence, l'ensemble des n-grammes ou des sous-arbres syntaxiques de taille bornée, etc. L'objectif est ensuite de trouver un ensemble $Z$ de taille $k$ qui maximise $f(Z)$, c.-à-d. qui garantit une couverture maximale pour chaque \emph{aspect}.

Dans cette définition, le choix d'un opérateur $\max$ s'appliquant uniformément à tous les \emph{aspects} est problématique et peut conduire à récupérer des phrases dans $Z$ qui n'ont que peu de pertinence pour la TA, voir annexe~\ref{appendix:illustration}.
En traduction, il est en effet plus important de couvrir certains \emph{aspects} que d'autres (par exemple des lexèmes rares dans une représentation sac-de-mots). Pour pallier ce problème, nous introduisons une nouvelle fonction sous-modulaire.

\begin{definition}[Couverture pondérée lissée]
	Soit $N \in \mathbb{N}$, $Z = \{z_1, \cdots, z_{|Z|}\}$ et $(v^{(n)}_i)_{n \in [1, N]}$ une séquence de poids réels pour $z_i \in Z$. Une \textbf{fonction de couverture pondérée lissée} de paramètre $\lambda \in [0, 1]$ est définie par: \vspace{-2em}
	\begin{align}
          \label{eq:definition-couverture-lambda}
          f(Z) = \sum_{n=1}^N \sum_{j = 1}^{|Z|} \lambda^{j-1} v^{(n)}_{g^{(n)}(j)},
          \end{align}
	avec $g^{(n)}$ une permutation telle que $i < j \Rightarrow v^{(n)}_{g^{(n)}(i)} \geq v^{(n)}_{g^{(n)}(j)}$, ordonnant les $z_i$ selon les $v^{(n)}_i$.
\end{definition}
La preuve de sa sous-modularité est en Annexe~\ref{appendix:preuve-sous-modularite}. Avec cette nouvelle définition, un \emph{aspects} $n$ déjà couvert continue de contribuer au calcul de $f$, mais avec un coefficient qui diminue exponentiellement. Pour $\lambda = 0$, $f$ correspond à la définition~\eqref{def:couv-ponderee} et, pour $\lambda = 1$, la fonction devient modulaire et sa maximisation revient à choisir indépendamment les $k$ exemples les plus couvrants.

\subsubsection{Maximisation de la couverture}

Maximiser $f$ pour $\lambda < 1$ est NP-complet \cite{krause-golovin-2014-submodular}. On ne connaît pas d'algorithme exact meilleur que l'énumération exhaustive.
L'algorithme glouton~\ref{algo:greedy-submodular} est la façon standard de maximiser une fonction sous-modulaire. Cependant, la combinaison linéaire figurant dans la définition de $f$ à l'équation~\eqref{eq:definition-couverture-lambda} devant être recalculée pour chaque candidat restant $z$, cet algorithme a une complexité temporelle $O(k^2N|T| + N|T|\log |T|)$ ainsi qu'un coût spatial lié aux permutations de tri $g$.

Pour améliorer la complexité, nous considérons l'algorithme~\ref{algo:w-greedy-submodular-approx} d'\citet{agrawal-etal-2023-context}, replacé ici dans le cadre de la théorie des fonctions sous-modulaires afin de garantir une borne inférieure.

\begin{algorithm}[H]
	\caption{Maximisation gloutonne d'une fonction sous-modulaire monotone}\label{algo:greedy-submodular}
	\KwData{source $x$, fonction sous-modulaire (de couverture de pondérée) $f$, ensemble de candidats $T$, nombre d'exemples souhaité $k$.}
	\KwResult{$Z$}

	$Z \gets \emptyset$ \;
	\While{$|Z| < k$} {
		$z^* \gets \underset{z \in T \setminus Z}{\argmax} f(Z \cup \{z\})$
		\;
		$Z \gets Z \cup \{ z^* \}$ \;
	}
	\Return{$Z$} \;
\end{algorithm}

\begin{algorithm}[H]
	\caption{Maximisation gloutonne d'une fonction de couverture pondérée lissée}\label{algo:w-greedy-submodular-approx}
	\KwData{source $x$, fonction sous-modulaire de couverture de $x$ pondérée lissée $f$ de facteur de sous-pondération $\lambda$, poids de couverture $v^{(n)}(z)$ pour $z\in T$ ensemble de candidats, nombre d'exemples souhaité $k$.}
	\KwResult{$Z$}

	$Z \gets \emptyset$ \;
	$W \gets (1, \dots, 1)$ \tcc*{$|W| = N$}

	\While{$|Z| < k$} {
		$z^* \gets \underset{z \in T \setminus Z}{\argmax} ~ W^T v(z)$ \tcc*{séléctionner $z^*$}
		$Z \gets Z \cup \{ z^* \}$ \;
		$I^* \gets \{ n : v^{(n)}(z^*) > 0 \}$ \tcc*{aspects couverts par $z^*$}

		\For{$n \in I^*$}{
			$W_n \gets \lambda W_n$ \tcc*{sous-pondération des éléments couverts}
		}

		\If{$\lambda = 0$ \And $W = (0, \dots, 0)$}{
			$W \gets (1, \dots, 1)$ \tcc*{réinitialiser $W$}
		}
	}
	\Return{$Z$} \;
\end{algorithm}

L'algorithme~\ref{algo:w-greedy-submodular-approx} a une complexité temporelle $O(kN|T|)$. Il se rapproche de l'algorithme~\ref{algo:greedy-submodular}, à la différence près que les $v^{(n)}_i$ ne sont pas triés.
Nous montrons en annexe~\ref{appendix:preuve-borne} que l'écart entre la solution retournée par l'algorithme~\ref{algo:w-greedy-submodular-approx} et la solution optimale est borné.

\section{Cadre Expérimental \label{sec:experiments}}

\subsection{Données}

Nous considérons des données anglais-français sur un panel de 6 corpus de domaine varié : ECB, EMEA, Europarl, JRC-Acquis, Ubuntu, Wikipedia\footnote{Les données sont en libre accès sur le site opus \href{https://opus.nlpl.eu}{https://opus.nlpl.eu}}. Cela correspond à un sous-ensemble des données utilisées par \citet{xu-etal-2023-integrating} dont nous reprenons la même partition entraînement/test.
Une caractéristique des données de test est qu'elles ont été partitionnées en deux ensembles de $1000$ phrases pour chaque domaine. Le premier ensemble (\textit{test-0.4}) contient des phrases sources pour lesquelles le plus proche voisin dans la TM (au sens de la similarité de Levenshtein\footnote{Définie pour deux chaînes $x$, $y$ par $1-\frac{d(x, y)}{\max |x|,|y|}$, avec $d$ la distance de Levenshtein.}) est à une distance entre $0,4$ et $0,6$. Pour le second (\textit{test-0.6}), le plus proche voisin a un score d'au moins $0,6$. Cela permet de différencier les comportements entre les échantillons avec des \textit{\og bonnes\fg{}} correspondances, et ceux avec des correspondances \textit{\og médiocres\fg{}}. Le tableau~\ref{tab:data-stats} résume les principales statistiques de ces corpus.

\begin{table*}[ht]
	\centering
	\scalebox{0.9}{
		\begin{tabular}{lrrrrrrrrrrrr}
			\hline
			domaine          & ECB  & EME  & Epp  & JRC  & Ubu & Wiki & \textbf{tout} \\ \hline
			taille           & 195k & 373k & 2,0M & 503k & 9k  & 803k & 3,9M         \\
			longueur moyenne & 29,2 & 16,7 & 26,6 & 28,8 & 5,2 & 19,6 & 21,0         \\
		\end{tabular}
	}
	\caption{\label{tab:data-stats}Nombre d'échantillons et longueur moyenne des phrases d'entraînement.}
\end{table*}

\subsection{Scores de Couverture}

Dans un premier temps, nous utilisons notre propre implémentation de BM25 \citep{robertson-and-sparck-1976-relevance} pour récupérer les $T=100$ meilleurs candidats. Ensuite, nous considérons plusieurs fonctions de couverture avec des scores et pondérations différents.

\textbf{Couverture sac-de-mot (SDM)} : Les \emph{aspects} sont les termes sac-de-mot $t_n$ de la source $x$ et le poids $v^{(n)}(z)$ correspond au minimum du nombre d'occurrences de $t_n$ dans le candidat $z$ et dans $x$. Cette notion correspond au \emph{rappel modifié}\footnote{"Modified recall", par analogie à la précision modifiée du score BLEU.} : les termes ne peuvent pas être couverts plus que leur nombre d'occurrences dans la source.

\textbf{Couverture 4-gramme ou moins (NGM)} : Les \emph{aspects} sont les 1-4-grammes $t_n$ de $x$, et $v^{(n)}(z)$ est le minimum du nombre d'occurrences dans $x$ et dans $z$.

\textbf{Couverture par distance de Levenshtein (DL)} : Les \emph{aspects} sont les indices de $x$, et $v^{(n)}(z)$ vaut $1$ ($0$ sinon) si et seulement si $x_n$ appartient à une sous-chaîne copiée en calculant la distance de Levenshtein entre $x$ et $z$. Puisqu'il peut exister plusieurs alignements optimaux, on peut soit calculer l'ensemble des sous-chaînes optimales et marginaliser pour chaque indice, soit échantillonner une solution et l'utiliser pour construire les $v^{(n)}(z)$.

Les $v^{(n)}(z)$ sont normalisés de deux manières différentes :

\textbf{Normalisation par cardinalité} : Pour SDM et NGM $v^{(n)}(z)$ est divisé par $N$ le nombre d'\emph{aspects}. Pour DL, on choisit de normaliser par le maximum de la taille de $x$ et $z$ afin de retrouver la formule classique quand $\lambda=1$.

\textbf{Normalisation par rareté (IDF)} : Chaque valeur $v^{(n)}(z)$ est normalisée avec des pondérations IDF, c.-à-d. que, pour SDM et DL, chaque mot $w$ de $x$ reçoit un poids $\operatorname{IDF}(w) /\sum_{w' \in x} \operatorname{IDF}(w)$. Pour NGM, l'indexation des n-grammes jusqu'à l'ordre 4 étant coûteuse, nous avons par simplification remplacé l'IDF de chaque n-gramme par la moyenne des IDF des termes qu'il contient.

Enfin, dans le cadre de l'introduction des fonctions de \textit{couverture pondérée lissées}, nous étudions différentes valeurs de $\lambda \in \{0; 0,2; 0,5; 1\}$.

\textbf{Recherche constrastive} Nous comparons aussi nos résultats à MMR \citep{cheng-etal-2022-neural} avec $\alpha=0,3$ qui fait un compromis entre pertinence et diversité des exemples.

\subsection{Métriques}

Nous calculons les scores BLEU \citep{papineni-etal-2002-bleu} avec SacreBLEU \citep{post-2018-call}\footnote{signature: \texttt{nrefs:1|case:mixed|eff:no|tok:13a|smooth:exp|version:2.1.0};}.
Nous étudions également trois métriques calculées en comparant les phrases cibles récupérées à la référence : la \emph{couverture}, la \emph{pertinence} et la \emph{longueur} des phrases. La couverture est calculée selon un rappel modifié, comme pour la fonction sous-modulaire unigramme avec $\lambda = 0$ sur les phrases tokénisées\footnote{Nous utilisons les scripts de Moses (\href{https://github.com/moses-smt/}{https://github.com/moses-smt/})}. La pertinence est la précision sac-de-mot, c.-à-d. la proportion de termes utiles dans les exemples rapportée à la somme des longueurs des exemples. La longueur est calculée sur les phrases tokenisées.

\section{Résultats \label{sec:resultats}}

\paragraph{Rôle du score choisi :} Nous effectuons une recherche pour les scores SDM, NGM et DL avec normalisation cardinale et $\lambda \in \{0; 0,2; 1\}$.
Les histogrammes des scores BLEU de la figure~\ref{fig:score-bleu} montrent :
(1) La supériorité de la distance de Levenshtein (DL) sur les deux autres scores, même si NGM reste compétitif sur certains domaines (voir tableau~\ref{tab:all-bleu-scores} dans l'annexe~\ref{appendix:complement} pour les résultats par domaine);
(2) À l'exception de NGM-0 (\textit{test-0.6}) maximiser la couverture avec $\lambda=0$ est en moyenne préférable et conduit à un gain de +0,1 (resp.\ +0,2) pour \textit{test-0.4} (resp. \textit{test-0.6}) par rapport à $\lambda=1$.

\begin{figure}
	\includegraphics[width=\textwidth]{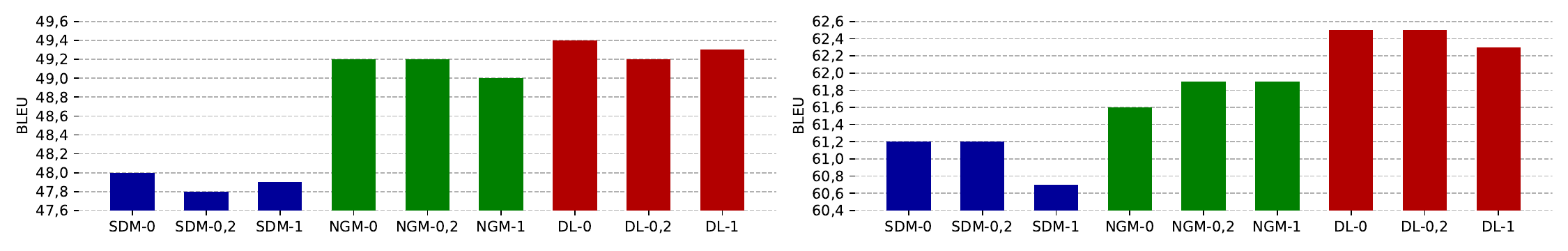}
	\caption{
		\label{fig:score-bleu}
		Scores BLEU moyens selon le score de la similarité DL, pour $\lambda \in \{0; 0,2; 1\}$ sur \textit{test-0.4} (gauche) et \textit{test-0.6} (droite).
	}
\end{figure}

\paragraph{Rôle de $\lambda$ :} Nous regardons l'impact de $\lambda$ pour DL (Distance de Levenshtein) figures~\ref{fig:lambda-couv-pert} et \ref{fig:lambda-long} en termes de \emph{couverture}, \emph{pertinence} et \emph{longueur}. 
Nous trouvons un compromis entre pertinence et couverture.
L'intervalle $[0,1; 0,9]$ est très stable.
Nous observons une singularité à $\lambda = 0$. La couverture est légèrement plus faible que $\lambda=0,1$.
Le cas $\lambda=1$ est singulier car ayant la couverture (resp. pertinence) la plus faible (resp. la plus élevée), ainsi que la plus faible longueur moyenne. Quant au score \textbf{BLEU} (voir figure~\ref{fig:lambda-bleu}), on observe un comportement différent entre les correspondances \textit{médiocres} (\textit{test-0.4}) et \textit{bonnes} (\textit{test-0.6}). $\lambda=0$ semble en général mieux sur \textit{test-0.4}, même si aucune tendance ne s'en dégage. Sur \textit{test-0.6}, on observe une courbe en cloche avec une inflexion à $\lambda=0$. La tendance semble indiquer que $\lambda=0,5$ produit les meilleurs résultats, et cela même sur \textit{test-0.4}. À noter qu'il y a des différences entre domaines (voir Annexe~\ref{appendix:complement}).
\begin{figure}
	\includegraphics[width=\textwidth]{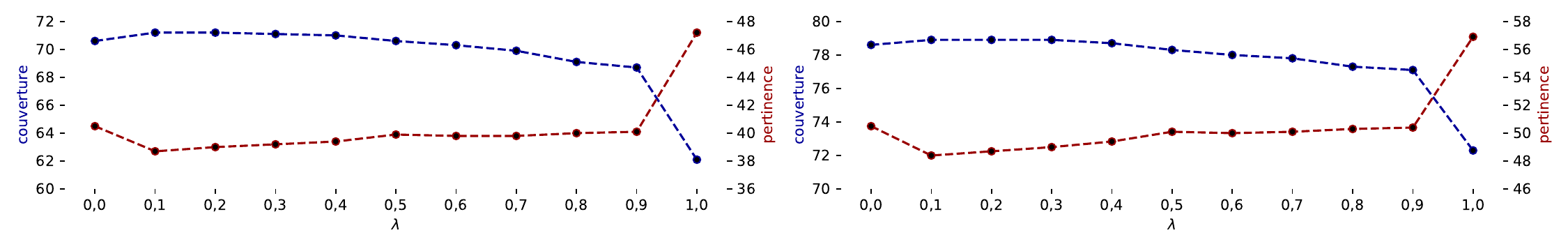}
	\caption{
		\label{fig:lambda-couv-pert}
		Couverture et pertinence moyenne selon la valeur de $\lambda$ pour DL sur \textit{test-0.4} (gauche) et \textit{test-0.6} (droite).
	}
\end{figure}
\begin{figure}

	\includegraphics[width=\textwidth]{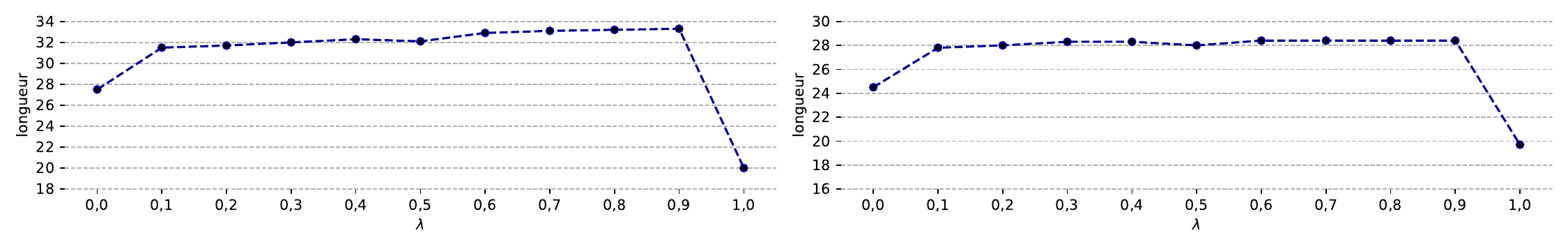}
	\caption{
		\label{fig:lambda-long}
		Longueur moyenne selon la valeur de $\lambda$ pour DL sur \textit{test-0.4} (gauche) et \textit{test-0.6} (droite).
	}
\end{figure}
\begin{figure}
	\includegraphics[width=\textwidth]{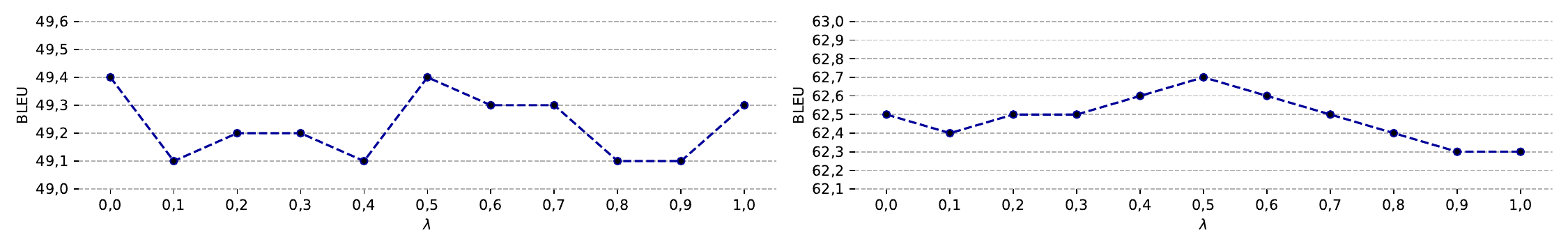}
	\caption{
		\label{fig:lambda-bleu}
		Score BLEU moyen en fonction de $\lambda$ pour DL sur \textit{test-0.4} (gauche) et \textit{test-0.6} (droite).
	}
\end{figure}

\paragraph{Comparaison avec MMR} Nous comparons les méthodes DL pour $\lambda \in \{0; 0,5; 1\}$
avec la méthode contrastive MMR. Sur la figure~\ref{fig:domain-bleu}, nous observons de très légères différences entre les méthodes de recherche. MMR est légèrement mieux (+0,1 BLEU) sur \textit{test-0.4}, mais DL-$0,5$ surpasse MMR de $0,1$ sur \textit{test-0.6}. Aucune des deux méthodes ne semble particulièrement supérieure.

\begin{figure}
	\includegraphics[width=\textwidth]{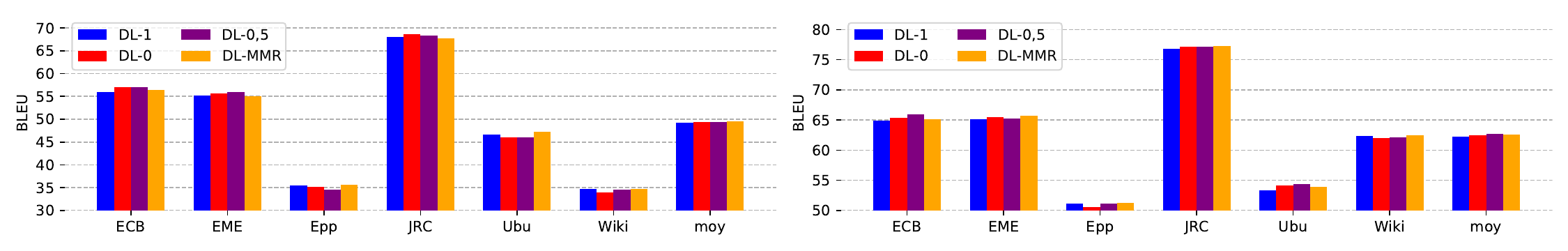}
	\caption{
		\label{fig:domain-bleu}
		Score BLEU avec $\lambda \in\{0; 0,5; 1\}$ et MMR sur \textit{test-0.4} (gauche) et \textit{test-0.6} (droite).
	}
\end{figure}

\paragraph{Rôle de la normalisation :} Par défaut, la normalisation se fait sur le nombre d'\emph{aspects} à couvrir, sauf pour DL où il s'agit du maximum entre la longueur de la source et de l'exemple. Lorsqu'on introduit une normalisation IDF, on observe un effet en moyenne négatif sur le score BLEU (voir figure~\ref{fig:idf-bleu}).
\begin{figure}
	\includegraphics[width=\textwidth]{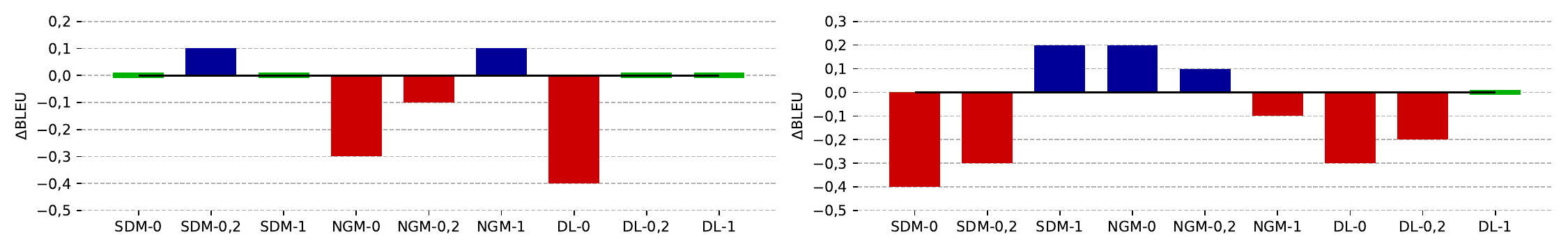}
	\caption{
		\label{fig:idf-bleu}
		Différence moyenne de BLEU entre normalisation IDF et cardinale, avec $\lambda \in\{0; 0,2; 1\}$ et SDM, NGM, DL sur \textit{test-0.4} (gauche) et \textit{test-0.6} (droite).
	}
\end{figure}
On peut expliquer ce comportement par le fait que le transformeur Multi-Levenshtein apprend à copier un maximum de tokens, même si ceux-ci sont fréquents. Autrement dit, il favorise la quantité à la qualité. Il vaut donc probablement mieux fournir des exemples plus couvrants, même s'il s'agit de termes communs.

\section{Conclusion}

Dans cet article, nous étudions comment optimiser la sélection d'un ensemble varié d'exemples dans une mémoire de traduction. Par analogie des travaux en recherche d'information, nous nous appuyons sur la théorie des fonctions sous-modulaires. Dans certaines configurations nous observons un léger gain avec l'utilisation de ce paradigme, avec de grandes variations selon le domaine. Une difficulté de cette approche est qu'elle sélectionne des exemples plus longs que la méthode de base, ce qui rend plus difficile leur édition conjointe et limite les améliorations des scores de traduction.

Dans le futur, nous souhaitons continuer à étudier la relation entre choix des exemples et modèle d'édition, par exemple en entraînant le modèle de recherche d'information conjointement avec le modèle de traduction. Une autre piste consiste à chercher à imposer une contrainte de longueur aux phrases couvrantes, afin qu'elles soient plus exploitables. Enfin, \mlevt{k} apprend à maximiser la couverture de la cible sans se soucier de la rareté des mots couverts. Peut-être vaut-il mieux privilégier l'utilité des mots couverts plutôt que leur nombre, en incitant le modèle à conserver des termes qu'il aurait du mal à regénérer.

\section{Remerciements}
Ce projet a été partiellement financé par l'ANR dans le cadre du projet TraLaLam (ANR-23-IAS1-0006). Il a également bénéficié des resources HPC/AI de GENCI-IDRIS (2022-AD011013583
et 2023-AD010614012).

\bibliographystyle{jeptaln2024}
\bibliography{../anthology,../custom}


\appendix

\section{Configuration du Modèle}
\label{appendix:config}

Nous utilisons le transformeur Multi-Levenshtein\footnote{disponible à \href{https://github.com/Maxwell1447/fairseq}{https://github.com/Maxwell1447/fairseq}.}. La dimension de plongement est $512$; la taille de la couche linéaire est de $2048$; le nombre de têtes est $8$; le nombre de couches de l'encodeur et de décodeur est $6$; les plongements sont tous partagés; le taux de dropout est $0,3$.

L'entraînement est fait avec un optimisateur Adam $(\beta_1, \beta_2)=(0,9; 0,98)$; un planificateur racine carrée inversée; un taux d'apprentissage de $5e^{-4}$; un lissage d'étiquette de $0,1$; une mise à jour à $10~000$; une précision flottante de $16$. Le nombre d'itérations est fixé à $60$k. La taille de lot et le nombre de GPU sont choisis pour avoir en moyenne $450$ échantillons par itération. Le modèle est pré-entraîné sur des données synthétiques décrites par \citet{bouthors-etal-2023-towards}.
Nous utilisons un vocabulaire joint de taille $32$k. Nous utilisons le réalignement et le réaffinage itératif avec une pénalité de $3$ sur le fait d'insérer $0$ à l'étape d'insertion, et un nombre maximum d'itérations de $10$. \cite{gu-etal-2019-levenshtein}.

Les données d'entraînement sont les $11$ domaines en-fr de \citet{bouthors-etal-2023-towards}, dont les $6$ domaines que nous avons choisis. Les exemples sont construits avec comme dans la configuration DL, c-à-d avec un préfiltrage des $100$ exemples avec les meilleurs scores BM25, puis les $3$ meilleurs exemples de similarité DL. Cependant, les phrases avec un score $<0,4$ sont retirées comme dans le papier original.

\section{Illustration}
\label{appendix:illustration}

Le tableau~\ref{tab:illustration} illustre le compromis entre \textbf{couverture} et \textbf{pertinence} (liée au score DL individuel). Ici, nous effectuons une RPS sur un micro corpus de 11 phrases plus ou moins similaires à la source. Le terme "vert" apparaissant uniquement dans une phrase très peu pertinente (dans le sens où peu de termes de la phrase sont dans la source), un $\lambda$ trop faible la récupérera pour compléter le dernier terme couvrable. Au contraire, un $\lambda$ trop élevé a tendance à récupérer des phrases similaires les unes aux autres, avec une haute pertinence, mais une faible couverture.

\begin{table}
	\begin{tabular}{c|c|c|l}
		méthode & couv & score DL & \textbf{source :} Le chat est assis sur le tapis vert du salon . \\ \hline\hline
		&&0,64& \underline{Le} \underline{chat} \underline{est} \underline{assis} \underline{sur} \underline{le} sol \underline{.} \\
		 DL-0&&0,27& J' ai acheté un nouveau \underline{tapis} pour \underline{le} \underline{salon} \underline{.} \\
		à&0,91&& Après une longue journée de marche , au crépuscule , je décide  \\
		DL-0,3&&0,08& enfin de me reposer à côté d' un grand rocher tout \underline{vert} de \\ 
		&&& mousse dans la forêt à côté du domaine de Courbetin \underline{.}  \\ \hline
		DL-0,4 &&0,64& \underline{Le} \underline{chat} \underline{est} \underline{assis} \underline{sur} \underline{le} sol \underline{.} \\
		à &0,82&0,27& J' ai acheté un nouveau \underline{tapis} pour \underline{le} \underline{salon} \underline{.} \\
		DL-0,6&&0,27& Regarde ce \underline{chat} , il \underline{est} \underline{assis} \underline{sur} \underline{le} comptoir \underline{.}  \\ \hline
		DL-0,7 &&0,64& \underline{Le} \underline{chat} \underline{est} \underline{assis} \underline{sur} \underline{le} sol \underline{.} \\
		à &0,64&0,27& Regarde ce \underline{chat} , il \underline{est} \underline{assis} \underline{sur} \underline{le} comptoir \underline{.} \\
		DL-0,9&&0,45& \underline{Le} \underline{chat} \underline{est} \underline{assis} à l' entrée \underline{.}  \\ \hline
		&&0,64& \underline{Le} \underline{chat} \underline{est} \underline{assis} \underline{sur} \underline{le} sol \underline{.} \\
		DL-1 &0,64&0,45& \underline{Le} \underline{chat} \underline{est} \underline{assis} à l' entrée \underline{.} \\
		&&0,36& \underline{Le} \underline{chat} est dans une boîte en carton \underline{.}  \\ \hline\hline
		&&0,64& \underline{Le} \underline{chat} \underline{est} \underline{assis} \underline{sur} \underline{le} sol \underline{.} \\
		DL-MMR &0,82&0,45& \underline{Le} \underline{chat} \underline{est} \underline{assis} à l' entrée \underline{.} \\
		&&0,27& J' ai acheté un nouveau \underline{tapis} pour \underline{le} \underline{salon} \underline{.} \\
	\end{tabular}
	\caption{\label{tab:illustration} Illustration des effets des paramètres de RPS ($\lambda$ et MMR) pour le score DL, avec un compromis entre couverture globale de la source (couv) et la proximité individuelle des phrases exemples. Les termes couvrants sont soulignés.}
\end{table}

\section{Démonstrations}
\label{appendix:preuves}

\subsection{Sous-modularité de la \textit{couverture pondérée lissée}}
\label{appendix:preuve-sous-modularite}

\begin{lemma}
	\label{lemma:def}
	Soit $f: 2^\Omega \to \mathbb R$. La propriété suivante découle de la définition de la sous-modularité de $f$ : $f$ sous-modulaire si et seulement si
	$\forall X \subseteq \Omega$, $\forall x_1, x_2 \in \Omega \setminus X$ s.t. $x_1 \neq x_2$: 
	\[
		f(X \cup \{x_1\}) + f(X \cup \{x_2\}) \geq f(X \cup \{x_1, x_2\}) + f(X)
	\]
\end{lemma}

Nous utilisons le lemme~\ref{lemma:def} pour montrer que la fonction de \textit{couverture pondérée lissée} \eqref{eq:definition-couverture-lambda} est sous-modulaire.
\begin{proof}
	Soit $f_n(Z)$ le terme dans l'équation \eqref{eq:definition-couverture-lambda} tel que $f(Z) = \sum_{n=1}^N f_n(Z)$ :

	\begin{equation}
		f_n(Z) = v_0 + \lambda v_1 + \cdots + \lambda^{K-1} v_{K-1},
	\end{equation}
	où les $v_i$ sont triés dans l'ordre décroissants, et $K = |Z|$.
	Soient $v$ et $v'$ la composante en $n$ de deux nouveaux éléments de $\Omega \setminus Z$. On considère que $v \geq v'$. On nomme $i$ et $j$ les indices tels que: $v_i \geq v \geq v_{i+1}$ et $v_j \geq v' \geq v_{j+1}$.

	\begin{align*}
		 & \bullet~ f_n(Z \cup \{v\}) =  v_0 + \cdots
		+ \lambda^i v_i + \lambda^{i+1} v + \lambda^{i+2} v_{i+1}
		+ \cdots + \lambda^{K} v_{K-1}                                                                                \\
		 & \bullet~ f_n(Z \cup \{v'\}) =  v_0 + \cdots
		+ \lambda^j v_j + \lambda^{j+1} v' + \lambda^{j+2} v_{j+1}
		+ \cdots + \lambda^{K} v_{K-1}                                                                                \\
		 & \bullet~ f_n(Z \cup \{v, v'\}) =  v_0 + \cdots
		+ \lambda^i v_i + \lambda^{i+1} v + \lambda^{i+2} v_{i+1}
		+ \cdots                                                                                                      \\
		 & ~~~~~~~~~~~~~~~~~~~~~~~~~~~~~~~~~~~~~~~~+ \lambda^{j+1} v_{j} + \lambda^{j+2} v' + \lambda^{j+3} v_{j+1}
		+ \cdots + \lambda^{K+1} v_{K-1}
	\end{align*}

	On simplifie en nommant
	$L = \sum_{l\leq i} \lambda^l v_l$,
	$M = \sum_{i < l \leq j} \lambda^l v_l$ et
	$R = \sum_{j < l} \lambda^l v_l$, ce qui donne:
	\begin{align*}
		 & \bullet~ f_n(Z \cup \{v\}) = L + \lambda^i v + \lambda (M + R)                                \\
		 & \bullet~ f_n(Z \cup \{v'\}) =  L + M + \lambda^j v + \lambda R                                \\
		 & \bullet~ f_n(Z \cup \{v, v'\}) = L + \lambda^i v + \lambda M + \lambda^{j+1} v' + \lambda^2 R
	\end{align*}

	Pour montrer que $f_n$ est sous-modulaire, il suffit de montrer que ce terme est positif :
	\[
		f_n(Z \cup \{v\}) + f_n(Z \cup \{v'\}) - f_n(Z \cup \{v, v'\}) - f_n(Z) = (1-\lambda) \left[ \lambda^j v' - (1-\lambda)R \right]
	\]
	Puisque $\lambda \in [0, 1]$, alors $1-\lambda > 0$.
	D'autre part :
	\begin{align*}
		R \leq \lambda^{j+1} v_{j+1}
		 & \Rightarrow -(1-\lambda) R \geq -(1-\lambda) \lambda^{j+1} v_{j+1}                            \\
		 & \Rightarrow \lambda^j v' - (1-\lambda)R  \geq \lambda^j v' -(1-\lambda) \lambda^{j+1} v_{j+1}
	\end{align*}

	$\lambda^j > \lambda^{j+1} (1 - \lambda) \geq 0$ et $v' \geq v_{j+1} \geq 0$, donc $\lambda^j v' \geq (1-\lambda) \lambda^{j+1} v_{j+1}$.

	$f_n(Z \cup \{v\}) + f_n(Z \cup \{v'\}) - f_n(Z \cup \{v, v'\}) - f_n(Z) \geq 0$ donc $f_n$ sous-modulaire.
	Puisque $f$ est une somme de fonction sous-modulaire, alors $f$ sous-modulaire.
\end{proof}

\subsection{Majoration de l'erreur d'approximation}
\label{appendix:preuve-borne}
\begin{theorem}
	Soit $f$ une fonction de couverture pondérée lissée de paramètre $\lambda$. Soit $Z$ la solution de l'algorithme~\ref{algo:w-greedy-submodular-approx} et $k$ le nombre d'exemples:
	\begin{equation}
		\label{eq:our-bound}
		f(Z) \geq \frac{1}{k} \left( \sum_{j = 0}^{k-1} \lambda^j \right) \max_{\bar Z : |\bar Z| = k} f(Z)
	\end{equation}
\end{theorem}

Si $\lambda = 1$, la solution est optimale. Si $\lambda = 0$, elle est au pire $\frac{1}{k}$ de l'optimum.
Dans le cas où les poids de couverture sont booléens ($0$ ou $1$), et $\lambda = 0$, l'algorithme~\ref{algo:w-greedy-submodular-approx} est équivalent au~\ref{algo:greedy-submodular} (hormis la partie de réinitialisation), ce qui nous permet de retomber sur la borne bien connue de $(1-1/e)$.
Nous montrons également que dans des cas limites spécifiques, cette borne peut être atteinte.

\begin{proof}
	L'équation \eqref{eq:our-bound} se démontre via les Lemmes \eqref{lemma:v_i_star-ineq} et \eqref{lemma:f_z_k-ineq}.
	Nous adoptons la notation suivante:
	$\Delta(v | Z) := f(Z \cup \{v\}) - f(Z) $.
	Soit $Z_k^*$ un maximiseur de $f$ de taille $k$, et $Z_k$, la solution construite par l'algorithme~\ref{algo:w-greedy-submodular-approx}, avec $Z_i$ (pour $i<k$) les itérations successives depuis $Z_0=\emptyset$. On note aussi, $Z_{i+1} = Z_i \cup \{v_{i+1}\}$.

	\begin{lemma}
		\label{lemma:v_i_star-ineq}
		\[
			\forall k > 0, f(Z_{k}^*) \leq \sum_{i=1}^k f(\{ v_i^* \}),
		\]
		où $v_i^*$ sont les éléments de $Z_{k}^*$ tels que $f(\{ v_i^* \})$ est une suite décroissante.
	\end{lemma}
	\begin{proof}
		\begin{align*}
			f(Z_k^*)
			 & = \sum_{i = 1}^k \Delta(v_i^* | \{v_1^*, \cdots, v_{i-1}^*\}) & \text{par téléscopage}               \\
			 & \leq \sum_{i = 1}^k \Delta(v_i^* | \emptyset)                 & \text{car } f \text{ sous-modulaire} \\
			 & = \sum_{i = 1}^k f(\{v_i^*\})                                 &                                      \\
		\end{align*}
	\end{proof}

	\begin{lemma}
		\label{lemma:delta-ineq}
		\[
			\forall i, \forall v \notin Z_i, \Delta(v|Z_i) \geq w_{i}^T v,
		\]
	\end{lemma}
	\begin{proof}
		Étant donné $n < N$, on nomme $z_i^{(n)}$ les éléments triés de $v_{[1:k]}^{(n)}$, tels que $z_1^{(n)} \geq \cdots \geq z_i^{(n)}$.
		$f$ s'exprime comme :
		\[
			f(Z_i) = \sum_{n=1}^N z_i^{(n)} \lambda^{i-1}
		\]

		Pour calculer $f(Z_i \cup \{v\})$ pour n'importe quel $v$,
		pour chaque $n$, $v^{(n)}$ est inséré dans la séquence ordonnée des
		$z_i^{(n)}$.
		Soit $m(n)$ la position d'insertion. Autrement dit, $z_1^{(n)} \geq \cdots \geq z_{m(n)}^{(n)} \geq v^{(n)} \geq z_{m(n)+1}^{(n)} \geq \cdots \geq z_i^{(n)}$.

		De même, on appelle $L^{(n)}$ et $R^{(n)}$ les parties gauche et droite de la somme :
		\[
			L^{(n)} = \sum_{i \leq m(n)} z_i^{(n)} \lambda^{i-1},
			~~~~~~~~~~~~~
			R^{(n)} = \sum_{i > m(n)} z_i^{(n)} \lambda^{i-1}
		\]

		Ainsi :
		\[
			f(Z_i \cup \{v\}) = \sum_{n=1}^N L^{(n)} + \lambda^{m(n)} v^{(n)} + \lambda R^{(n)}.
		\]

		\begin{align*}
			\Delta(v|Z_i)
			                                                & = f(Z_i \cup \{v\}) - f(Z_i)                                                     & \\
			                                                & = \sum_{n=1}^N \left[ L^{(n)} + \lambda^{m(n)} v^{(n)} + \lambda R^{(n)} \right]
			- \sum_{n=1}^N \left[ L^{(n)} + R^{(n)} \right] &                                                                                    \\
			                                                & = \sum_{n=1}^N \lambda^{m(n)} v^{(n)}
			- (1-\lambda) \sum_{n=1}^N R^{(n)}              &                                                                                    \\
		\end{align*}

		Maintenant, à propos de $w_{i-1}^T v$, on note $c(n)$ le nombre de $v_j^{(n)}$ strictement positifs ($1 \leq j < i$). Par construction, $w_{i-1}^{(n)} = \lambda^{c(n)}$, d'où:
		\[
			w_{i-1}^T v = \sum_{n=1}^N \lambda^{c(n)} v^{(n)}.
		\]

		\begin{align*}
			                                                                             & \Delta(v|Z_i) - w_{i-1}^T v                                   \\
			                                                                             & = \sum_{n=1}^N \lambda^{m(n)} v^{(n)}
			- (1-\lambda) \sum_{n=1}^N R^{(n)} - \sum_{n=1}^N \lambda^{c(n)} v^{(n)}     &                                                               \\
			                                                                             & = \sum_{n=1}^N (\lambda^{m(n)} - \lambda^{c(n)}) v^{(n)}
			- (1-\lambda) \sum_{n=1}^N \sum_{i > m(n)} z_i^{(n)} \lambda^{i-1}           &                                                               \\
			                                                                             & \geq \sum_{n=1}^N (\lambda^{m(n)} - \lambda^{c(n)}) v^{(n)}
			- (1-\lambda) \sum_{n=1}^N \sum_{m(n) < i \leq c(n)} z_i^{(n)} \lambda^{i-1} & \text{ par construction de } c(n)                             \\
			                                                                             & \geq \sum_{n=1}^N (\lambda^{m(n)} - \lambda^{c(n)}) v^{(n)}
			-  \sum_{n=1}^N v^{(n)} (1-\lambda) \sum_{m(n) < i \leq c(n)} \lambda^{i-1}  & \text{pour ces } i, z_i^{(n)} \leq v^{(n)}                    \\
			                                                                             & = \sum_{n=1}^N (\lambda^{m(n)} - \lambda^{c(n)}) v^{(n)}
			-  \sum_{n=1}^N v^{(n)} (\lambda^{m(n)} - \lambda^{c(n)})                    & \text{somme téléscopique}                                     \\
			                                                                             & = 0                                                         & \\
		\end{align*}
	\end{proof}

	\begin{lemma}
		\label{lemma:f_z_k-ineq}
		\[
			\forall k > 0, f(Z_{k}) \geq \sum_{i=1}^k f(\{ v_i^* \}) \lambda^{i-1},
		\]
		où $v_i^*$ sont éléments de $Z_{k}^*$ tels que $f(\{ v_i^* \})$ est une suite décroissante.
	\end{lemma}
	\begin{proof}
		\begin{align*}
			f(Z_k)
			 & = \sum_{i=1}^k \Delta(v_i | Z_{i-1}) &                                                \\
			 & \geq \sum_{i=1}^k w_{i-1}^T v_i      & \text{grâce au Lemme \eqref{lemma:delta-ineq}} \\
		\end{align*}
		Soit $i \leq k$. Prouvons que $w_{i-1}^T v_i \geq \lambda^{i-1} f(\{v_i^*\}) = \lambda^{i-1} 1^T v_i^*$. Il y a deux cas:
		\begin{itemize}
			\item Si $v_i^* \notin Z_i$, alors par construction : $w_{i-1}^T v_i \geq w_{i-1}^T v_i^*$. Puisque chaque $w_{i-1}^{(n)} \geq \lambda^{i-1}$, alors $w_{i-1}^T v_i \geq \lambda^{i-1} 1^T v_i^*$.
			\item Si $v_i^* \in Z_i$, on sait que $\exists j<i$ tel que $v_j^* \notin Z_i$. Donc $w_{i-1}^T v_i \geq w_{i-1}^T v_j^* \geq \lambda^{i-1} 1^T v_j^*$. Puique $f(\{ v_i^* \})$ suite décroissante, $1^T v_j^* \geq 1^T v_i^*$. Par conséquent, $w_{i-1}^T v_i \geq \lambda^{i-1} 1^T v_i^*$.
		\end{itemize}
		D'où, $f(Z_k) \geq \sum_{i=1}^k f(\{ v_i^* \}) \lambda^{i-1}$.
	\end{proof}

	En combinant les Lemmes \eqref{lemma:v_i_star-ineq} et \eqref{lemma:f_z_k-ineq}, on obtient :
	\[
		\frac{f(Z_k)}{f(Z_k^*)} \geq \frac{\sum_{i=1}^k f(\{ v_i^* \}) \lambda^{i-1}}{\sum_{i=1}^k f(\{ v_i^* \})}.
	\]
	Grâce à l'inégalité de Tchebychev pour les sommes :
	\[
		\frac{f(Z_k)}{f(Z_k^*)} \geq \frac{\sum_{i=1}^k \lambda^{i-1}}{k} = \frac{1}{k} \frac{1 - \lambda^k}{1 - \lambda},
	\]
	D'où la borne.
\end{proof}

Enfin, nous pouvons montrer que la borne est atteinte dans un scénario, signifiant qu'il s'agit bien de la meilleure borne.
\begin{proof}
	Pour cela, nous supposons $N>k$.

	On suppose que l'ensemble de candidats est $T = (z_1, \cdots, z_N, e_1, \cdots e_N)$, tel que
	\[
		\left\{
		\begin{array}{l}
			\forall i, z_i = (\frac{1}{N}, \cdots, \frac{1}{N}) \\
			\forall i, e_i^{(n)} = \mathbbm{1}(i = n)
		\end{array}
		\right.
	\]

	Si l'algorithme s'obstine à choisir seulement des $z_j$, alors $w_i = (\lambda^i, \cdots, \lambda^i)$. Ce faisant, $\forall z \in T, w_i^T z = \lambda^i$. Par conséquent, l'algorithme peut choisir de manière équivalente n'importe quel $z \in T$ pour la prochaine itération. Si il ne choisit que des $z_j$ pour $Z_k$, alors :
	\[
		f(Z_k) = 1 + \lambda + \cdots + \lambda^{k-1} = \frac{1 - \lambda^k}{1 - \lambda}
	\]
	D'autre part, le meilleur ensemble $Z_k^*$ de $k$ éléments est constitué uniquement de $e_j$ :
	\[
		f(Z_k^*) = k
	\]
	Le ratio est celui de l'équation \eqref{eq:our-bound}.
\end{proof}

\section{Compléments}
\label{appendix:complement}

Pour avoir une vision complète des résultats, nous présentons les tableaux de couverture, pertinence et longueur (tableau~\ref{tab:all-metrics-scores}), ainsi que des scores BLEU par domaine (tableau~\ref{tab:all-bleu-scores} et tableau~\ref{tab:all-lambda-bleu-scores}).

Les domaines offrent une grande variabilité dont il est difficile d'extraire des tendances. Notons cependant que Europarl et Wikipedia ne bénéficient pas d'une plus grande couverture car ce sont les cas où $\lambda=1$ qui donnent le meilleur BLEU. C'est le contraire pour ECB et JRC-Aquis.

\begin{table}[!ht]
	\centering
	\begin{tabular}{l|lll|lll}
		            & \multicolumn{3}{c}{\textit{test-0.4}} & \multicolumn{3}{c}{\textit{test-0.6}}                                  \\ \hline
		recherche   & couv.                                 & perti.                                & long. & couv. & perti. & long. \\ \hline\hline
		SDM-0       & 70,8                                  & 36,7                                  & 35,6  & 77,4  & 46,0   & 31,3  \\
		SDM-0,2     & 71,0                                  & 36,5                                  & 36,1  & 77,5  & 45,7   & 31,7  \\
		SDM-1       & 66,7                                  & 38,1                                  & 37,8  & 75,2  & 47,7   & 32,1  \\
		SDM-IDF-0   & 70,4                                  & 36,3                                  & 34,8  & 77,2  & 45,5   & 31,4  \\
		SDM-IDF-0,2 & 70,8                                  & 36,4                                  & 35,0  & 77,4  & 45,5   & 31,5  \\
		SDM-IDF-1   & 68,5                                  & 37,7                                  & 37,6  & 75,3  & 46,9   & 32,4  \\
		NGM-0       & 70,3                                  & 38,5                                  & 30,5  & 77,4  & 47,2   & 28,3  \\
		NGM-0,2     & 70,3                                  & 39,1                                  & 31,0  & 77,4  & 48,1   & 28,6  \\
		NGM-1       & 66,6                                  & 40,4                                  & 32,4  & 74,6  & 49,7   & 29,1  \\
		NGM-IDF-0   & 70,4                                  & 38,3                                  & 30,3  & 77,5  & 46,9   & 28,1  \\
		NGM-IDF-0,2 & 70,6                                  & 38,9                                  & 30,8  & 77,5  & 47,7   & 28,4  \\
		NGM-IDF-1   & 66,9                                  & 40,2                                  & 32,3  & 74,8  & 49,1   & 29,3  \\
		DL-0        & 70,6                                  & 40,5                                  & 27,5  & 78,6  & 50,5   & 24,5  \\
		DL-0,2      & 71,2                                  & 39,0                                  & 31,7  & 78,9  & 48,7   & 28,0  \\
		DL-1        & 62,1                                  & 47,2                                  & 20,0  & 72,3  & 56,9   & 19,7  \\
		DL-IDF-0    & 70,6                                  & 39,5                                  & 28,1  & 78,6  & 49,5   & 25,1  \\
		DL-IDF-0,2  & 70,9                                  & 38,7                                  & 30,9  & 78,8  & 48,3   & 27,8  \\
		DL-IDF-1    & 62,1                                  & 47,2                                  & 20,0  & 72,3  & 56,9   & 19,7  \\
		DL-MMR      & 64,3                                  & 46,1                                  & 20,4  & 73,6  & 56,3   & 19,9  \\
	\end{tabular}
	\caption{
		\label{tab:all-metrics-scores}Couverture, pertinence et longueur moyennes selon le score de recherche choisi, avec ou sans normalisation IDF, et $\lambda \in \{0; 0,2; 1\}$.}

\end{table}

\begin{table}[!ht]
	\centering
	\begin{tabular}{llllllll}
		\                 & ECB  & EME  & Epp  & JRC  & Ubu  & Wiki & moy  \\ \hline\hline
		\textit{test-0.4} &      &      &      &      &      &      &      \\ \hline
		SDM-0             & 55,0 & 54,1 & 33,8 & 66,8 & 45,7 & 32,5 & 48,0 \\
		SDM-0,2           & 54,7 & 54,4 & 33,2 & 66,8 & 45,4 & 32,2 & 47,8 \\
		SDM-1             & 55,1 & 53,1 & 34,3 & 66,9 & 45,0 & 33,1 & 47,9 \\
		SDM-IDF-0         & 55,1 & 54,5 & 33,5 & 67,1 & 45,9 & 32,0 & 48,0 \\
		SDM-IDF-0,2       & 55,0 & 54,4 & 33,6 & 67,0 & 45,4 & 32,3 & 47,9 \\
		SDM-IDF-1         & 55,5 & 53,4 & 33,1 & 67,3 & 44,7 & 33,2 & 47,9 \\
		NGM-0             & 56,1 & 56,6 & 34,7 & 68,2 & 46,0 & 33,6 & 49,2 \\
		NGM-0,2           & 56,2 & 56,2 & 34,9 & 68,3 & 45,8 & 33,9 & 49,2 \\
		NGM-1             & 56,0 & 55,7 & 35,3 & 67,8 & 45,3 & 33,8 & 49,0 \\
		NGM-IDF-0         & 55,7 & 56,4 & 34,4 & 68,0 & 45,5 & 33,5 & 48,9 \\
		NGM-IDF-0,2       & 56,0 & 56,5 & 34,7 & 68,1 & 45,4 & 33,8 & 49,1 \\
		NGM-IDF-1         & 55,8 & 55,6 & 34,9 & 67,8 & 45,7 & 34,7 & 49,1 \\
		DL-0              & 57,0 & 55,6 & 35,2 & 68,7 & 46,1 & 33,9 & 49,4 \\
		DL-0,2            & 56,6 & 55,6 & 35,0 & 68,4 & 45,7 & 33,8 & 49,2 \\
		DL-1              & 55,9 & 55,2 & 35,5 & 68,0 & 46,6 & 34,7 & 49,3 \\
		DL-IDF-0          & 56,1 & 55,5 & 34,5 & 68,1 & 46,1 & 33,5 & 49,0 \\
		DL-IDF-0,2        & 56,3 & 55,4 & 34,6 & 68,0 & 46,9 & 34,1 & 49,2 \\
		DL-IDF-1          & 55,9 & 55,2 & 35,4 & 68,0 & 46,7 & 34,8 & 49,3 \\ \hline
		\hline
		\textit{test-0.6} &      &      &      &      &      &      &      \\ \hline
		SDM-0             & 64,5 & 62,5 & 50,1 & 75,7 & 53,6 & 60,8 & 61,2 \\
		SDM-0,2           & 64,1 & 62,5 & 50,1 & 75,8 & 53,9 & 60,6 & 61,2 \\
		SDM-1             & 63,8 & 61,9 & 50,1 & 74,8 & 52,3 & 61,2 & 60,7 \\
		SDM-IDF-0         & 64,9 & 61,6 & 49,4 & 76,0 & 52,8 & 59,9 & 60,8 \\
		SDM-IDF-0,2       & 64,7 & 61,9 & 49,2 & 75,8 & 53,6 & 60,2 & 60,9 \\
		SDM-IDF-1         & 64,7 & 61,8 & 50,0 & 75,5 & 52,0 & 61,2 & 60,9 \\
		NGM-0             & 64,2 & 63,6 & 51,0 & 76,7 & 53,0 & 61,1 & 61,6 \\
		NGM-0,2           & 64,9 & 63,7 & 51,2 & 76,8 & 53,2 & 61,5 & 61,9 \\
		NGM-1             & 64,8 & 63,8 & 51,2 & 76,5 & 52,5 & 62,9 & 61,9 \\
		NGM-IDF-0         & 64,7 & 64,0 & 51,1 & 76,8 & 53,1 & 61,4 & 61,8 \\
		NGM-IDF-0,2       & 65,7 & 63,5 & 51,0 & 76,8 & 53,5 & 61,5 & 62,0 \\
		NGM-IDF-1         & 65,5 & 63,8 & 50,8 & 76,3 & 52,6 & 62,0 & 61,8 \\
		DL-0              & 65,4 & 65,5 & 50,6 & 77,1 & 54,2 & 62,0 & 62,5 \\
		DL-0,2            & 65,4 & 65,8 & 51,0 & 77,0 & 53,9 & 61,9 & 62,5 \\
		DL-1              & 64,9 & 65,1 & 51,2 & 76,8 & 53,3 & 62,4 & 62,3 \\
		DL-IDF-0          & 65,3 & 65,2 & 50,7 & 77,1 & 53,7 & 61,3 & 62,2 \\
		DL-IDF-0,2        & 65,7 & 64,9 & 51,0 & 76,9 & 53,9 & 61,2 & 62,3 \\
		DL-IDF-1          & 64,9 & 65,1 & 51,2 & 76,8 & 53,3 & 62,4 & 62,3 \\
	\end{tabular}
	\caption{
		\label{tab:all-bleu-scores}Score BLEU selon le score de recherche choisi, avec ou sans normalisation IDF, et $\lambda \in \{0; 0,2; 1\}$.}
\end{table}

\begin{table}[!ht]
	\centering
	\begin{tabular}{llllllll}
		$\lambda$         & ECB  & EME  & Epp  & JRC  & Ubu  & Wiki & moy  \\ \hline\hline
		\textit{test-0.4} &      &      &      &      &      &      &      \\ \hline
		0                 & 57,0 & 55,6 & 35,2 & 68,7 & 46,1 & 33,9 & 49,4 \\
		0,1               & 56,4 & 55,3 & 34,7 & 68,4 & 46,3 & 33,7 & 49,1 \\
		0,2               & 56,6 & 55,6 & 35,0 & 68,4 & 45,7 & 33,8 & 49,2 \\
		0,3               & 56,5 & 55,4 & 35,0 & 68,6 & 45,9 & 33,6 & 49,2 \\
		0,4               & 56,9 & 55,5 & 34,8 & 68,2 & 45,6 & 33,9 & 49,1 \\
		0,5               & 57,0 & 56,0 & 34,5 & 68,4 & 46,0 & 34,5 & 49,4 \\
		0,6               & 57,1 & 55,8 & 34,7 & 68,2 & 45,7 & 34,5 & 49,3 \\
		0,7               & 56,8 & 55,6 & 35,1 & 68,2 & 45,5 & 34,4 & 49,3 \\
		0,8               & 56,6 & 55,2 & 35,0 & 68,2 & 45,0 & 34,3 & 49,1 \\
		0,9               & 56,5 & 55,4 & 35,2 & 68,1 & 45,2 & 34,2 & 49,1 \\
		1                 & 55,9 & 55,2 & 35,5 & 68,0 & 46,6 & 34,7 & 49,3 \\
		\hline
		\hline
		\textit{test-0.6} &      &      &      &      &      &      &      \\ \hline
		0                 & 65,4 & 65,5 & 50,6 & 77,1 & 54,2 & 62,0 & 62,5 \\
		0,1               & 65,5 & 65,4 & 50,8 & 77,0 & 53,8 & 61,8 & 62,4 \\
		0,2               & 65,4 & 65,8 & 51,0 & 77,0 & 53,9 & 61,9 & 62,5 \\
		0,3               & 65,6 & 65,5 & 51,2 & 77,1 & 53,8 & 62,0 & 62,5 \\
		0,4               & 65,9 & 65,3 & 51,2 & 77,1 & 54,2 & 62,1 & 62,6 \\
		0,5               & 66,0 & 65,2 & 51,2 & 77,2 & 54,4 & 62,1 & 62,7 \\
		0,6               & 66,0 & 65,0 & 51,1 & 77,1 & 54,2 & 62,4 & 62,6 \\
		0,7               & 65,7 & 64,7 & 51,1 & 76,9 & 54,1 & 62,3 & 62,5 \\
		0,8               & 65,6 & 65,1 & 50,8 & 76,8 & 53,7 & 62,4 & 62,4 \\
		0,9               & 65,5 & 65,2 & 50,8 & 76,8 & 53,5 & 62,4 & 62,3 \\
		1                 & 64,9 & 65,1 & 51,2 & 76,8 & 53,3 & 62,4 & 62,3 \\
	\end{tabular}
	\caption{
		\label{tab:all-lambda-bleu-scores}Score BLEU selon la valeur de $\lambda$ pour DL avec normalisation cardinale.}
\end{table}
\end{document}